\documentclass[10pt]{article}

\usepackage{graphicx,amsmath,amsfonts,amssymb,bm,hyperref,url,breakurl,epsfig,epsf,color,fullpage,MnSymbol,mathbbol,fmtcount,algorithmic,algorithm,semtrans,cite,caption,subcaption,multirow}
  
  \usepackage{mathtools}

\usepackage{titlesec}

\usepackage{tikz}
\usepackage{pgfplots}
\usetikzlibrary{pgfplots.groupplots}

\titleformat{\paragraph}
{\normalfont\normalsize\bfseries}{\theparagraph}{1em}{}
\titlespacing*{\paragraph}
{0pt}{3.25ex plus 1ex minus .2ex}{1.5ex plus .2ex}

\usepackage{movie15}

\usepackage{cite}
\usepackage{caption}
\usepackage[bottom,hang,flushmargin]{footmisc} 

\usepackage{hyperref}
\definecolor{darkred}{RGB}{150,0,0}
\definecolor{darkgreen}{RGB}{0,150,0}
\definecolor{darkblue}{RGB}{0,0,200}
\hypersetup{colorlinks=true, linkcolor=darkred, citecolor=darkgreen, urlcolor=darkblue}

\newtheorem{theorem}{Theorem}[section]
\newtheorem{lemma}[theorem]{Lemma}
\newtheorem{corollary}[theorem]{Corollary}

\newtheorem{definition}[theorem]{Definition}

\newcommand{\eps}{\varepsilon}

\newcommand{\beq}{\begin{equation}}
\newcommand{\eeq}{\end{equation}}

\newcommand{\nn}{\nonumber}

\newcommand{\A}{{\mtx{A}}}

\newcommand{\B}{{{\mtx{B}}}}
\newcommand{\Gb}{{\mtx{G}}}
\newcommand{\F}{{\mtx{F}}}

\newcommand{\Db}{{\mtx{D}}}
\newcommand{\Sb}{{\mtx{S}}}
\newcommand{\Rb}{{\mtx{R}}}
\newcommand{\Kh}{{\hat{K}}}
\newcommand{\Iden}{{\bf{I}}}

\newcommand{\z}{{\mtx{z}}}

\newcommand{\Cc}{\mathcal{C}}

\newcommand{\Bc}{\mathcal{B}}
\newcommand{\Sc}{\mathcal{S}}

\newcommand{\Nn}{\mathcal{N}}

\newcommand{\vb}{\mtx{v}}
\newcommand{\w}{\mtx{w}}

\newcommand{\li}{\left<}
\newcommand{\ri}{\right>}

\newcommand{\ab}{\vct{a}}
\newcommand{\bb}{\vct{b}}
\newcommand{\ub}{\vct{u}}
\newcommand{\h}{\vct{h}}
\newcommand{\g}{\vct{g}}

\newcommand{\onenorm}[1]{\left\|#1\right\|_{\ell_1}}
\newcommand{\infnorm}[1]{\left\|#1\right\|_{\ell_\infty}}
\newcommand{\twonorm}[1]{\left\|#1\right\|_{\ell_2}}

\newcommand{\x}{{\vct{x}}}

\newcommand{\y}{{\vct{y}}}

\definecolor{emmanuel}{RGB}{255,127,0}

\newcommand{\R}{\mathbb{R}}
\newcommand{\Pro}{\mathbb{P}}

\newcommand{\sgn}[1]{\textrm{sgn}(#1)}

\newcommand{\E}{\operatorname{\mathbb{E}}}

\newcommand{\vct}[1]{\bm{#1}}
\newcommand{\mtx}[1]{\bm{#1}}
\newcommand{\ml}{m_{lin}}

\newcommand{\bke}{{\bar{K}_{\eps}}}

\newcommand{\order}[1]{{\cal{O}}( #1)}
\newcommand{\tn}{\twonorm}
\newcommand{\on}{\onenorm}
\newcommand{\ino}{\infnorm}

\numberwithin{equation}{section} 

\def \endprf{\hfill {\vrule height6pt width6pt depth0pt}\medskip}
\newenvironment{proof}{\noindent {\bf Proof} }{\endprf\par}

\newcommand*\samethanks[1][\value{footnote}]{\footnotemark[#1]}

\newcommand{\ang}{\text{ang}}
\newcommand{\ham}[2]{{{\|#1,#2\|_H}}}

\title{Near-Optimal Bounds for Binary Embeddings of Arbitrary Sets} 
\author{Samet Oymak\thanks{Simons Institute, UC Berkeley}~\thanks{Department of Electrical Engineering and Computer Science, UC Berkeley}\quad \quad\quad\quad\quad Benjamin Recht\samethanks[2]~\thanks{Department of Statistics, UC Berkeley, Berkeley CA}}
\date{}
\begin{document}
\maketitle

\begin{abstract} We study embedding a subset $K$ of the unit sphere to the Hamming cube $\{-1,+1\}^m$. We characterize the tradeoff between distortion and sample complexity $m$ in terms of the Gaussian width $\omega(K)$ of the set. For subspaces and several \emph{structured-sparse} sets we show that Gaussian maps provide the optimal tradeoff $m\sim \delta^{-2}\omega^2(K)$, in particular for $\delta$ distortion one needs $m\approx\delta^{-2}{d}$ where $d$ is the subspace dimension. For general sets, we provide sharp characterizations which reduces to $m\approx{\delta^{-4}}{\omega^2(K)}$ after simplification. We provide improved results for local embedding of points that are in close proximity of each other which is related to locality sensitive hashing. We also discuss faster binary embedding where one takes advantage of an initial sketching procedure based on Fast Johnson-Lindenstauss Transform. Finally, we list several numerical observations and discuss open problems. 
\end{abstract}

\section{Introduction}
Thanks to applications in high-dimensional statistics and randomized linear algebra, recent years saw a surge of interest in dimensionality reduction techniques. As a starting point, Johnson-Lindenstauss lemma considers embedding a set of points in a high dimensional space to a lower dimension while approximately preserving pairwise $\ell_2$-distances. While $\ell_2$ distance is the natural metric for critical applications, one can consider embedding with arbitrary norms/functions. In this work, we consider embedding a high-dimensional set of points $K$ to the hamming cube in a lower dimension, which is known as binary embedding. Binary embedding is a natural problem arising from quantization of the measurements and is connected to $1$-bit compressed sensing as well as locality sensitive hashing \cite{wang2014hashing,plan2013robust,boufounos20081,kushilevitz2000efficient,thrampoulidis2015lasso,plan2014high,ai2014one}. In particular, given a subset $K$ of the unit sphere in $\R^n$, and a dimensionality reduction map $\A\in\R^{m\times n}$, we wish to ensure that
\beq
\left|\frac{1}{m}\|\sgn{\A\x},\sgn{\A\y}\|_H-\ang(\x,\y)\right|\leq \delta.\label{main bound}
\eeq
Here $\ang(\x,\y)\in [0,1]$ is the geodesic distance between the two points which is obtained by normalizing the smaller angle between $\x,\y$ by $\pi$. $\|\ab,\bb\|_H$ returns the Hamming distance between the vectors $\ab$ and $\bb$ i.e. the number of entries for which $\ab_i\neq \bb_i$. Relation between the geodesic distance and Hamming distance arises naturally when one considers a Gaussian map. For a vector $\g\sim\Nn(0,\Iden_n)$, $\|\sgn{\g^T\x},\sgn{\g^T\y}\|_H$ has a Bernoulli distribution with mean $\ang(\x,\y)$. This follows from the fact that $\g$ corresponds to a uniformly random hyperplane and $\x,\y$ lies on the opposite sides of the hyperplane with probability $\ang(\x,\y)$. Consequently, when $\A$ has independent standard normal entries, a standard application of Chernoff bound shows that \eqref{main bound} holds with probability $1-\exp(-2\delta^2m)$ for a given $(\x,\y)$ pair. This argument is sufficient to ensure binary embedding of finite set of points in a similar manner to Johnson-Lindenstrauss Lemma i.e.~$m\geq \order{\frac{\log p}{\delta^2}}$ samples are sufficient to ensure \eqref{main bound} for a set of $p$ points. 

{\bf{Embedding for arbitrary sets:}} When dealing with sets that are not necessarily discrete, highly nonlinear nature of the sign function makes the problem more challenging. Recent literature shows that for a set $K$ with infinitely many points, the notion of size can be captured by mean width $\omega(K)$
\beq
\omega(K)=\E_{\g\sim\Nn(0,\Iden_n)}[\sup_{\vb\in K}\g^T\vb].\label{gauss width}
\eeq
For the case of linear embedding, where the goal is preserving $\ell_2$ distances of the points, it is well-known that $m\sim\order{\delta^{-2}\omega^2(K)}$ samples ensure $\delta$-distortion for a large class of random maps. The reader is referred to a growing number of literature for results on $\ell_2$ embedding of continuous sets \cite{bourgain2013toward,oymak2015isometric,mendelson2008uniform,rudelson2008sparse}. There is much less known about the the randomized binary embeddings of arbitrary sets and the problem was initially studied by Plan and Vershynin \cite{plan2014dimension} when the mapping $\A$ is Gaussian. Remarkably, the authors show that embedding \eqref{main bound} is possible if $m\gtrsim \delta^{-6}{\omega^2(K)}{}$. This result captures the correct relation between the sample complexity $m$ and the set size $\omega(K)$ in a similar manner to $\ell_2$-embedding. On the other hand, it is rather weak when it comes to the dependence on the distortion $\delta$. Following from the result on finite set of points, intuitively, this dependence should be $\delta^{-2}$ instead of $\delta^{-6}$. For the set of sparse vectors Jacques et al.~obtained near-optimal sample complexity $\order{\delta^{-2}s\log(n/\delta)}$ \cite{jacques2013robust}. Related to this, Yi et al. \cite{yi2015binary} considered the same question where they showed that $\delta^{-4}$ is achievable for simpler sets such as intersection of a subspace and the unit sphere. More recently in \cite{jacques2015small} Jacques studies a related problem in which samples $\A\x$ are uniformly quantized instead of being discretized to $\{+1,-1\}$. There is also a growing amount of literature related to binary embedding and one-bit compressed sensing \cite{boufounos20081,thrampoulidis2015lasso,plan2014high,ai2014one}. In this work, we significantly improve the distortion dependence bounds for the binary embedding of arbitrary sets. For structured set such as subspaces, sparse signals, and low-rank matrices we obtain the optimal dependence $\order{\delta^{-2}\omega^2(K)}$. In particular, we show that a $d$ dimensional subspace can be embedded with $m=\order{\delta^{-2}{d}{}}$ samples. For general sets, we find a bound on the embedding size in terms of the covering number and local mean-width which is a quantity always upper bounded by mean width. While the specific bound will be stated later on, in terms of mean-width we show that $m=\order{\delta^{-4}{\omega^2(K)}{}}$ samples are sufficient for embedding. For important sets such as subspaces, set of sparse vectors and low-rank matrices our results have a simple message: Binary embedding works just as well as linear embedding.

{\bf{Locality sensitive properties:}} In relation to locality sensitive hashing \cite{andoni2006near,charikar2002similarity}, we might want $\x,\y$ to be close to each other if and only if $\sgn{\A\x},\sgn{\A\y}$ are close to each other. Formally, we consider the following questions regarding a binary embedding.
\begin{itemize}
\item How many samples do we need to ensure that for all $\x,\y$ satisfying $\ang(\x,\y)>\delta$, we have that $\frac{1}{m}\|\sgn{\A\x},\sgn{\A\y}\|_H\gtrsim\order{\delta}$?
\item How many samples do we need to ensure that for all $\x,\y$ satisfying $\frac{1}{m}\|\sgn{\A\x},\sgn{\A\y}\|_H>\delta$, we have that $\ang(\x,\y)\gtrsim\order{\delta}$?
\end{itemize}
These questions are less restrictive compared to requiring \eqref{main bound} for all $\x,\y$ as a result they should  intuitively require less sample complexity. Indeed, we show that the distortion dependence can be improved by a factor of $\delta$ e.g. for a subspace one needs only $m=\order{\frac{\omega^2(K)}{\delta}\log\delta^{-1}}$ samples and for general sets one needs only $m=\order{\frac{\omega^2(K)}{\delta^3}\log\delta^{-1}}$ samples. Our distortion dependence is an improvement over the related results of Jacques \cite{jacques2015small}.

{\bf{Sketching for binary embedding:}} While providing theoretical guarantees are beneficial, from an application point of view, efficient binary embedding is desirable. A good way to accomplish this is to make use of matrices with fast multiplication. Such matrices include subsampled Hadamard or Fourier matrices as well as super-sparse ensembles. Efficient sketching matrices found applications in a wide range of problems to speed up machine learning algorithms \cite{cohen2014dimensionality,woodruff2014sketching,gittens2013revisiting}. For binary embedding, recent works in this direction include \cite{yi2015binary,yu2014circulant} which have limited theoretical guarantees for discrete sets. In Section \ref{sec:sketch} we study this procedure for embedding arbitrary sets and provide guarantees to achieve faster embeddings.

\section{Main results on binary embedding}

Let us introduce the main notation that will be used for the rest of this work. $\Sc^{n-1}$ and $\Bc^{n-1}$ are the unit Euclidian sphere and ball respectively. $\Nn(\boldsymbol{\mu},\boldsymbol{\Sigma})$ denotes a Gaussian vector with mean $\boldsymbol{\mu}$ and covariance $\boldsymbol{\Sigma}$. $\Iden_n$ is the identity matrix of size $n$. A matrix with independent $\Nn(0,1)$ entries will be called standard Gaussian matrix. Gaussian width $\omega(\cdot)$ is defined above in \eqref{gauss width} and will be used to capture the size of our set of interest $K\subset\R^n$. $c,C,\{c_i\}_{i=0}^\infty,\{C_i\}_{i=0}^\infty$ denote positive constants that may vary from line to line. $\ell_p$ norm is denoted by $\|\cdot\|_{\ell_p}$ and $\|\cdot\|_0$ returns the sparsity of a vector. Given a set $K$ denote $\text{cone}(K)=\{\alpha \vb~|~\alpha\geq 0,~\vb\in K\}$ by $\Kh$ and define the binary embedding as follows.

\begin{definition} [$\delta$-binary embedding] \label{bin embed}Given $\delta\in (0,1)$, $f:\R^n\rightarrow\{0,1\}^m$ is a $\delta$-binary embedding of the set $\Cc$ if for all $\x,\y\in\Cc$, we have that
\beq
|\frac{1}{m}\|f(\x),f(\y)\|_H-\ang(\x,\y)|\leq \delta.\nn 
\eeq
\end{definition}

With this notation, we are in a position to state our first result which provides distortion bounds on binary embedding in terms of the Gaussian width $\omega(K)$. 

\begin{theorem} \label{thm main 1} Suppose $\A\in \R^{m\times n}$ has independent $\Nn(0,1)$ entries. Given set $K\subset\Sc^{n-1}$ and a constant $1>\delta>0$, there exists positive absolute constants $c_1,c_2$ such that the followings hold.
\begin{itemize}
\item {\bf{$\Kh$ is a subspace:}} Whenever $m\geq c_1 \frac{\omega^2(K)}{\delta^2}$, with probability $1-\exp(-c_2\delta^2 m)$, $\A$ (i.e.~$f:\x\rightarrow\sgn{\A\x}$) is a $\delta$-binary embedding for $K$.
\item {\bf{$K$ is arbitrary:}} Whenever $m\geq c_1 \frac{\omega^2(K)}{\delta^4}\log\frac{1}{\delta}$, with probability $1-\exp(-c_2 \delta^2m)$, $\A$ is a $\delta$-binary embedding for $K$.
\end{itemize}
\end{theorem}
For arbitrary sets, we later on show that the latter bound can be improved to $\frac{\omega^2(K)}{\delta^4}$. This is proven by applying a sketching procedure and is deferred to Section \ref{sec:sketch}.

The next theorem states our results on the local properties of binary embedding namely it characterizes the behavior of the embedding in a small neighborhood of points.
\begin{definition} [Local $\delta$-binary embedding] \label{loc bin embed}Given $\delta\in (0,1)$, $f:\R^n\rightarrow\{0,1\}^m$ is a local $\delta$-binary embedding of the set $\Cc$ if there exists constants $c_{up}, c_{low}$ such that
\begin{itemize} 
\item For all $\x,\y\in\Cc$ satisfying $\tn{\x-\y}\geq \delta$: $|\frac{1}{m}\|f(\x),f(\y)\|_H-\ang(\x,\y)|\geq c_{low}\delta$.
\item For all $\x,\y\in\Cc$ satisfying $\tn{\x-\y}\leq \delta/\sqrt{\log \delta^{-1}}$: $|\frac{1}{m}\|f(\x),f(\y)\|_H-\ang(\x,\y)|\leq c_{up}\delta$.
\end{itemize}
\end{definition}
\begin{theorem} \label{thm main 2} Suppose $\A\in \R^{m\times n}$ is a standard Gaussian matrix. Given set $K\subset\Sc^{n-1}$ and a constant $1>\delta>0$, there exist positive absolute constants $c_1,c_2$ such that the followings hold.
\begin{itemize}
\item {\bf{$\Kh$ is a subspace:}} Whenever $m\geq c_1 \frac{\omega^2(K)}{\delta}\log\frac{1}{\delta}$, with probability $1-\exp(-c_2 \delta m)$, $\A$ is a local $\delta$-binary embedding for $K$.
\item {\bf{$K$ is arbitrary:}} Whenever $m\geq c_1 \frac{\omega^2(K)}{\delta^3}\log\frac{1}{\delta}$, with probability $1-\exp(-c_2 \delta m)$, $\A$ is a local $\delta$-binary embedding for $K$.
\end{itemize}
\end{theorem}

We should emphasize that the second statements of Theorems \ref{thm main 1} and \ref{thm main 2} are simplified versions of a better but more involved result. To state this, we need to introduce two relevant definitions.
\begin{itemize}
\item {\bf{Covering number:}} Let $N_\eps$ be the $\eps$-covering number of $K$ with respect to the $\ell_2$ distance.
\item {\bf{Local set:}} Given $\alpha>0$, define the local set $K_\eps$ to be
\beq
K_\eps=\{\vb\in\R^n~\big|~\vb=\ab-\bb,~\ab,\bb\in K,~\tn{\ab-\bb}\leq \eps\}.\nn
\eeq
\end{itemize}
In this case, our tighter sample complexity estimate is as follows.
\begin{theorem}\label{thm main 3} Suppose $\A\in \R^{m\times n}$ is a standard Gaussian matrix. Given set $K\subset\Sc^{n-1}$ and a constant $1>\delta>0$, there exist positive absolute constants $c,c_1,c_2$ such that the followings hold. Setting $\eps=c\delta/\sqrt{\log\delta^{-1}}$, 
\begin{itemize}
\item whenever $m\geq   c_1\max\{\delta^{-2}\log N_\eps,\delta^{-3}\omega^2(K_\eps)\}$, with probability $1-\exp(-c_2 \delta^2 m)$, $\A$ is a $\delta$-binary embedding for $K$.
\item whenever $m\geq   c_1\max\{\delta^{-1}\log N_\eps,\delta^{-3}\omega^2(K_\eps)\}$, with probability $1-\exp(-c_2 \delta m)$, $\A$ is a local $\delta$-binary embedding for $K$.
\end{itemize}
\end{theorem}
\noindent {\bf{Implications for structured sets:}} This result is particularly beneficial for important low-dimensional sets for which we have a good understanding of the local set $K_\eps$ and covering number $N_\eps$. Such sets include
\begin{itemize}
\item $\Kh$ is a subspace,
\item $\Kh$ is a union of subspaces for instance set of $d$-sparse signals (more generally signals that are sparse with respect to a dictionary),
\item $\Kh$ is set of rank-$d$ matrices in $\R^{n_1\times n_2}$ where $n=n_1\times n_2$,
\item $\Kh$ is the set of $(d,l)$ group-sparse signals. This scenario assumes small entry-groups $\{G_i\}_{i=1}^N\subset\{1,2,\dots,n\}$ which satisfy $|G_i|\leq l$. A vector is considered group sparse if its nonzero entries lie in at most $d$ of these groups. Sparse signals is a special case for which each entry is a group and group size $l=1$.
\end{itemize}
These sets are of fundamental importance for high-dimensional statistics and machine learning. They also have better distortion dependence. In particular, as a function of the model parameters (e.g. $d,n,l$) there exists a number $C(K)$ such that (e.g.~\cite{candes2011tight})
\beq
\log N_\eps\leq C(K)\log\frac{1}{\eps},~\omega^2(K_\eps)\leq  \eps^2C(K)\label{structured dependence}.
\eeq
For all of the examples above, we either have that $C(K)\sim \omega^2(K)$ or the simplified closed form upper bounds of these quantities (in terms of $d,l,n$) are the same. \cite{candes2011tight,vershynin2010introduction}. Consequently, Theorem \ref{thm main 3} ensures that for such low-dimensional sets we have the near-optimal dependence for binary embedding namely $m\sim \frac{\omega^2(K)}{\delta^2}\log\frac{1}{\delta}$. This follows from the improved dependencies $\omega^2(K_\eps)\sim \delta^2\omega^2(K)/\log\delta^{-1}$ and $\log N_\eps\sim \omega^2(K)\log\delta^{-1}$.

While these bounds are near-optimal they can be further improved and can be matched to the bounds for linear embedding namely $m\sim \delta^{-2}\omega^2(K)$. The following theorem accomplishes this goal and allows us to show that binary embedding performs as good as linear embedding for a class of useful sets..
\begin{theorem} \label{optimal struct}Suppose \eqref{structured dependence} holds for all $\eps>0$. Then, there exists $c_1,c_2$ such that if $m>c_1\delta^{-2} C(K)$, $\A$ is a $\delta$-binary embedding for $K$ with probability $1-\exp(-c_2\delta^2m)$.
\end{theorem}

The rest of the paper is organized as follows. In Section \ref{small dev} we give a proof of Theorem \ref{thm main 1}. Section \ref{bin chain} specifically focuses on obtaining optimal distortion guarantees for structured sets and provides the technical argument for Theorem \ref{optimal struct}. Section \ref{sec local} provides a proof for Theorem \ref{thm main 2} which is closely connected to the proof of Theorem \ref{thm main 1}. Section \ref{sec:sketch} focuses on sketching and fast binary embedding techniques for improved guarantees. Numerical experiments and our observations are presented in Section \ref{sec:numerics}.
\section{Main proofs}\label{small dev}

Let $K$ be an arbitrary set over the unit sphere, $\bke$ be an $\eps$ covering of $K$ and $N_\eps=|\bke|$. Due to Sudakov Minoration there exists an absolute constant $c>0$ for which $\log N_\eps\leq \frac{c\omega^2(K)}{\eps^2}$. This bound is often suboptimal for instance, if $\Kh$ is a $d$ dimensional subspace, then we have that
\beq
\log N_\eps\leq \omega^2(K)\log( \frac{c}{\eps})\nn
\eeq
where $\omega^2(K)\approx d$. Recall that our aim is to ensure that for near optimal values of $m$, all $\x,\y\in K$ obeys
\beq
|\ang(\x,\y)-\frac{1}{m}\|\sgn{\A\x},\sgn{\A\y}\|_H|\leq \delta\label{rule}.
\eeq
This task is simpler when $K$ is a finite set. In particular, when $\A$ has standard normal entries we have that $\|\sgn{\A\x},\sgn{\A\y}\|_H$ is sum of $m$ i.i.d.~Bernoulli random variables with mean $\ang(\x,\y)$. This ensures
\begin{align}
&\ang(\x,\y)=\E[\frac{1}{m}\|\sgn{\A\x},\sgn{\A\y}\|_H],\nn\\
&\Pro(|\ang(\x,\y)-\frac{1}{m}\|\sgn{\A\x},\sgn{\A\y}\|_H|> \delta)\leq \exp(-2\delta^2m).\nn
\end{align}
Consequently, as long as we are dealing with finite sets one can use a union bound. This argument yields the following lemma for $\bke$.
\begin{lemma} \label{outer lemma}Assume $m\geq \frac{2}{\delta^2} \log N_\eps$. Then, with probability $1-\exp(-\delta^2 m)$ all points $\x,\y$ of $\bke$ obeys \eqref{rule}.
\end{lemma}

Using this as a starting point, we will focus on the effect of continuous distortions to move from $\bke$ to $K$. The following theorem considers the second statement of Definition \ref{loc bin embed} and states our main result on local deviations.
\begin{theorem} \label{prop second} Suppose $\A\in\R^{m\times n}$ is a standard Gaussian matrix. Given $1>\delta>0$, pick $c>0$ to be a sufficiently large constant, set $c\eps= \delta(\log\frac{1}{\delta})^{-1/2}$ and assume that
\beq
m\geq  c\max\{\delta^{-3}\omega^2(K_\eps),\frac{1}{\delta}\log N_\eps\}.\nn
\eeq
Then, with probability $1-2\exp(-\delta m/64)$, for all $\x,\y\in K$ obeying $\tn{\y-\x} \leq \eps$ we have
 \beq
m^{-1}\|\A\x,\A\y\|_H\leq \delta.\nn
\eeq
\end{theorem}
This theorem is stated in terms of Gaussian width of $K_\eps$ and the covering number $N_\eps$. Making use of the fact that $\omega(K_\eps)\leq 2\omega(K)$ and $\log N_\eps\leq c\frac{\omega^2(K)}{\eps^2}$ (and similar simplification for subspaces), we arrive at the following corollary.

\begin{corollary} \label{for num 2}Suppose $\A\in\R^{m\times n}$ is a standard Gaussian matrix.
\begin{itemize} 
\item When $K$ is a general set, set
\beq
m\geq  c\delta^{-3}\log\frac{1}{\delta}\omega^2(K),\nn
\eeq
\item When $\Kh$ is a $d$-dimensional subspace, set
\beq
m\geq  c\delta^{-1}\log\frac{1}{\delta}d,\nn
\eeq
\end{itemize}
 for sufficiently large constant $c>0$. Then, with probability $1-2\exp(-\delta m/64)$, for all $\x,\y\in K$ obeying $\tn{\y-\x} \leq c^{-1}\delta(\log\frac{1}{\delta})^{-1/2}$ we have
 \beq
m^{-1}\|\A\x,\A\y\|_H\leq \delta.\nn
\eeq
\end{corollary}

\subsection{Preliminary results for the proof}
We first define relevant quantities for the proof. Given a vector $\x$ let $\tilde{\x}$ denote the vector obtained by sorting absolute values of $\x$ decreasingly. Define
\beq
\|\x\|_{k+}=\sum_{i=1}^k \tilde{\x},~\|\x\|_{k-}=\sum_{i=n-k+1}^n \tilde{\x}.\nn
\eeq
In words, $k+$ and $k-$ functions returns the $\ell_1$ norms of the top $k$ and bottom $k$ entries respectively. The next lemma illustrates why $k+$ and $k-$ functions are useful for our purposes.

\begin{lemma} \label{lemma basic}Given vectors $\x$ and $\y$, suppose $\|\x\|_{k-}> \|\y\|_{k+}$. Then, we have that
\beq
\|\x,\x+\y\|_H< k.\nn
\eeq
\end{lemma}
\begin{proof} Suppose at $i$th location $\sgn{\x_i}\neq \sgn{\x_i+\y_i}$. This implies $|\y_i|\geq |\x_i|$. If $\|\x,\x+\y\|_H\geq k$ it implies that there is a set $S\subset\{1,2,\dots,n\}$ of size $k$. Over this subset $\onenorm{\y_S}\geq \onenorm{\x_S}$ which yields
\beq
\|\y\|_{k+}\geq \onenorm{\y_S}\geq \on{\x_S}\geq \|\x\|_{k-}.\nn
\eeq
This contradicts with the initial assumption.
\end{proof}
The next two subsections obtains bounds on $k+$ and $k-$ functions of a Gaussian vector in order to be able to apply Lemma \ref{lemma basic} later on.
\subsubsection{Obtain an estimate on $k+$}
The reader is referred to Lemma \ref{append k+} for a proof of the following result.
\begin{lemma} \label{gauss k+}Suppose $\g\sim\Nn(0,\Iden_n)$ and $0<\delta<1$ is a sufficiently small constant. Set $k=\delta n$. There exists constants $c,C>0$ such that for $n>C\delta^{-1}$ we have that
\beq
\E[\|\g\|_{k+}]=\E[\sum_{i=1}^{\delta n} \tilde{\g}_i]\leq c\delta n\sqrt{\log\frac{1}{\delta}}.\nn
\eeq
\end{lemma}

The next lemma provides an upper bound for $\sup_{\x\in\Cc}\|\A\x\|_{k+}$ in expectation.
\begin{lemma} \label{slepian}Let $\A$ be a standard Gaussian matrix and $\Cc\subset\R^n$. Define $\text{diam}(\Cc)=\sup_{\vb\in\Cc}\tn{\vb}$. Set $k=\delta m$ for a small constant $0<\delta<1$. Then, there exists constants $c,C>0$ such that for $m>C\delta^{-1}$
\beq
\E[\sup_{\x\in\Cc}\|\A\x\|_{k+}]\leq c\text{diam}(\Cc)m\delta\log\frac{1}{\delta}+\sqrt{m\delta}\omega(\Cc).\nn
\eeq
\end{lemma}
\begin{proof} Let $S_k=\{\vb\in\R^n~\big|~\|\vb\|_0\leq k,~\ino{\vb}\leq 1\}$. Observe that
\beq
\sup_{\x\in\Cc}\|\A\x\|_{k+}=\sup_{\x\in\Cc,\vb\in S_k}\vb^T\A\x.\nn
\eeq
Applying Slepian's Lemma \ref{slep var} (also see \cite{OnGordon} or \cite{Gor}), we find that, for $\g\sim\Nn(0,\Iden_m),\h\sim\Nn(0,\Iden_n),g\sim \Nn(0,1)$,
\begin{align}
\E[\sup_{\x\in\Cc}\sup_{\vb\in S_k}\vb^T\A\x]&\leq \E[\sup_{\x\in\Cc,\vb\in S_k}\tn{\x}\vb^T\g+\x^T\h\tn{\vb}]+\E[\sup_{\x\in\Cc,\vb\in S_k}\tn{\x}\tn{\vb}|g|]\nn\\
&\leq \text{diam}(\Cc)\E[\|\g\|_{k+}]+\sqrt{k}\omega(\Cc)+\text{diam}(\Cc)\sqrt{k}\nn\\
&\leq c\text{diam}(\Cc)m\delta\sqrt{\log\frac{1}{\delta}}+\sqrt{m\delta}\omega(\Cc)\nn
\end{align}
which is the advertised result. For the final line we made use of the estimate obtained in Lemma \ref{gauss k+}.
\end{proof}
\subsubsection{Obtain an estimate on $k-$}
As the next step we obtain an estimate of $\|\g\|_{k-}$ for $\g\sim\Nn(0,\Iden_m)$ by finding a simple deviation bound.

\begin{lemma}\label{lemma minus} Suppose $\g\sim\Nn(0,\Iden_m)$. Let $\gamma_{\alpha}$ be the number for which $\Pro(|g|\leq \gamma_{\alpha})=\alpha$ where $g\sim\Nn(0,1)$ (i.e.~the inverse cumulative density function of $|g|$). Then,
\beq
\Pro(\|\g\|_{\delta m-}\geq m\frac{\delta\gamma_{\delta/2}}{4})\geq 1-\exp(-\frac{\delta m}{32}).\nn
\eeq
$\gamma_{\alpha}$ trivially obeys $\gamma_{\alpha}\geq \sqrt{\frac{\pi}{2}}\alpha$. This yields
\beq
\Pro(\|\g\|_{\delta m-}\geq \frac{m\delta^2}{8})\geq 1-\exp(-\frac{\delta m}{32}).\nn
\eeq
\end{lemma}
\begin{proof}
We show this by ensuring that, with high probability, among the bottom $m\delta$ entries at least $m\delta/4$ of them are greater than $\gamma_{\delta/2}$. This is a standard application of Chernoff bound. Let $a_i$ be a random variable which is $1$ if $|\g_i|\leq \beta$ and $0$ else. Then, if $s:=\sum_{i=1}^n a_i\leq \frac{3\delta}{4}$, it would imply that at most $\frac{3\delta}{4}$ of entries of $|\g|$ are less than $\beta$. We will argue that, this is indeed the case with the advertised probability for $\beta=\gamma_{\delta/2}$. For $\beta=\gamma_{\delta/2}$, we have that $\Pro(a_i=1)=\frac{\delta}{2}$. Hence, applying a standard Chernoff bound, we obtain
\beq
\Pro(\sum_{i=1}^m a_i>m\frac{3\delta}{4})\leq \exp(-\frac{\delta m}{32}).\nn
\eeq
With this probability, we have that out of the bottom $\delta m$ entries of $\g$ at least $\frac{\delta}{4}$ of them are greater than or equal to $\gamma_{\delta/2}$. This implies $\|\g\|_{\delta m-}\geq m\frac{\delta\gamma_{\delta/2}}{4}$. To conclude we use the standard fact that inverse absolute-gaussian cumulative function obeys $\gamma_{\alpha}\geq \sqrt{\frac{\pi}{2}}\alpha$.
\end{proof}

\subsection{Proof of Theorem \ref{prop second}}

We are in a position to prove Theorem \ref{prop second}. Without losing generality, we may assume $\delta<\delta'$ where $\delta'$ is a sufficiently small constant. The result for $\delta\geq \delta'$ is implied by the case $\delta=\delta'$.


\begin{proof} Set $\eps>0$ to be $c'\eps=\delta/\sqrt{\log\delta^{-1}}$ for a constant $c'\geq 1$ to be determined. Using Lemma \ref{lemma basic}, the proof can be reduced to showing the following claim: Under given conditions, all $\x\in \bke,~\y\in K$ satisfying $\tn{\x-\y}\leq \eps$ obey
\beq
\|\A\x\|_{\delta m-}\geq \|\A(\y-\x)\|_{\delta m+}.\label{relation 1}
\eeq 

To show this, we shall apply a union bound. For a particular $\x\in \bke$ applying Lemma \ref{lemma minus}, we know that
\beq
\Pro(\|\A\x\|_{\delta m-}\geq \frac{m\delta^2}{8})\geq 1-\exp(-\frac{\delta m}{32}).\label{part 1}
\eeq
Using a union bound, we find that if $N_\eps<\exp(\frac{\delta m}{64})$, with probability $1-\exp(-\frac{\delta m}{64})$, all $\x\in \bke$ obeys the relation above. This requires $ m\geq \frac{64}{\delta}\log N_{\eps}$ which is satisfied by assumption.

We next show that given $\x\in \bke$ and $\y\in K$ that is in the $\eps$ neighborhood of $\x$ we have that
\beq
 \|\A(\y-\x)\|_{\delta m+}\leq  \frac{m\delta^2}{8}.\nn
\eeq
Observe that $\x-\y\in K_\eps$, consequently, applying Lemma \ref{slepian} and using $\eps\sqrt{\delta m}$-Lipschitzness of $f(\A)=\sup_{\vb\in K_\eps}\|\A\vb\|_{\delta m+}$, with probability $1-\exp(-\delta m)$, for an absolute constant $c>0$ we have that
\begin{align}
\sup_{\vb\in K_\eps}\|\A\vb\|_{\delta m+}&\leq c\eps \delta\sqrt{\log \frac{1}{\delta}} m+\sqrt{\delta m}\omega(K_\eps)\label{part 2}.
\end{align}

Following \eqref{part 1} and \eqref{part 2}, we simply need to determine the conditions for which 
\beq
c\eps \delta\sqrt{\log \frac{1}{\delta}} m+\sqrt{m\delta}\omega(K_\eps)\leq \frac{m\delta^2}{8}.\nn
\eeq
This inequality holds if
\begin{itemize}
\item $cc'\leq 1/16$,
\item $m\geq 256\delta^{-3}\omega^2(K_\eps)$.
\end{itemize}
To ensure this we can pick $m,c'$ to satisfy the conditions while keeping initial assumptions intact. With these, \eqref{relation 1} is guaranteed to hold concluding the proof.
\end{proof}
\subsubsection{Proof of Corollary \ref{for num 2}}
We need to substitute the standard covering and local-width bounds to obtain this result while again setting $c\eps=\delta/\sqrt{\log\delta^{-1}}$. For general $K$, simply use the estimates $\omega(K_\eps)\leq 2\omega(K)$, $\log N_\eps \leq c\omega(K)^2/\eps^2$. When $\Kh$ is a $d$ dimensional subspace we use the estimates $\omega(K_\eps)\leq \eps \sqrt{d}$ and $\log N_\eps \leq cd\log\eps^{-1}$ and use the fact that $\log \eps^{-1}\sim\log\delta^{-1}$. 



\subsection{Proof of Theorem \ref{thm main 3}: First Statement}
Proof of the first statement of Theorem \ref{thm main 3} follows by merging the discrete embedding and ``small deviation'' arguments. We restate it below.
\begin{theorem}  \label{for num 3}Suppose $\A\in\R^{m\times n}$ is a standard Gaussian matrix. Set $c\eps=\delta/\sqrt{\log\delta^{-1}}$. $\A$ provides a $\delta$-binary embedding of $K$ with probability $1-\exp(-c_2\delta^2m)$ if the number of samples satisfy the bound
\beq
m\geq  c_1\max\{\delta^{-2}\log N_\eps,\delta^{-3}\omega^2(K_\eps)\}.\nn
\eeq
\end{theorem}
\begin{proof}
Given $\delta>0$, $c\eps=\delta/\sqrt{\log\delta^{-1}}$, applying Lemma \ref{outer lemma} we have that all $\x',\y'\in \bke$ obeys \eqref{rule} whenever $m\geq 2\delta^{-2}\log N_\eps$. Next whenever the conditions of Theorem \ref{prop second} are satisfied all $\x,\y\in K,\x',\y'\in \bke$ satisfying $\tn{\x'-\x}\leq \eps, \tn{\y'-\y}\leq \eps$, obeys
\beq
m^{-1}\|\A\x',\A\x\|_H\leq \delta,~ m^{-1}\|\A\y',\A\y\|_H\leq \delta.\nn
\eeq
Combining everything and using the fact that $\ang(\x,\x'),\ang(\y,\y')\leq \pi^{-1}\eps$ (as angular distance is locally Lipschitz around $0$ with respect to the $\ell_2$ norm), under given conditions with probability $1-\exp(-C\delta^2m)$ we find that
\begin{align}
|m^{-1}\|\A\x,\A\y\|_H-\ang(\x,\y)|&\leq |\ang(\x,\y)-\ang(\x',\y)|+|\ang(\x',\y)-\ang(\x',\y')|\nn\\
&+|m^{-1}\|\A\x',\A\y'\|_H-\ang(\x',\y')|+m^{-1}\|\A\x',\A\x\|_H+m^{-1}\|\A\y',\A\y\|_H\nn\\
&\leq 2\pi^{-1}\eps+2\delta+\delta \leq 4\delta.\nn
\end{align}
For the second line we used the fact that
\beq
\|\x,\y\|_H-\|\x',\y\|_H\leq \|\x,\x'\|_H.\nn
\eeq
Using the adjustment $\delta'\leftrightarrow 4\delta$ we obtain the desired continuous binary embedding result. The only additional constraint to the ones in Theorem \ref{prop second} is the requirement $m\geq 2\delta^{-2}\log N_\eps$ which is one of the assumptions of Theorem \ref{thm main 1} thus we can conclude with the result.
\end{proof}
Observe that this result gives following bounds for $\delta$-binary embedding.
\begin{itemize}
\item For an arbitrary $K$, $m\geq \omega^2(K)\delta^{-4}\log\delta^{-1}$ samples are sufficient. This yields the corresponding statement of Theorem \ref{thm main 1}.
\item When $\Kh$ is a $d$ dimensional subspace, $m\geq \omega^2(K)\delta^{-2}\log\delta^{-1}$ samples are sufficient.
\end{itemize}
More generally, one can plug improved covering bounds for the scenarios $\Kh$ is the set of $d$ sparse vectors or set of rank $d$ matrices to show that for these sets $m\geq \omega^2(K)\delta^{-2}\log\delta^{-1}$ is sufficient. A useful property of these sets are the fact that $K-K$ is still low-dimensional for instance if $\Kh$ is the set of $d$ sparse vectors then the elements of $K-K$ are at most $2d$ sparse.

\subsubsection{Proof of Theorem \ref{optimal struct} via improved structured embeddings}
So far our subspace embedding bound requires a sample complexity of $\order{d\delta^{-2}\log\delta^{-1}}$ which is slightly suboptimal compared to the linear Johnson-Lindenstrauss embedding with respect to $\ell_2$ norm. To correct this, we need a more advanced discrete embedding result which requires a more involved argument. In particular we shall use Theorem \ref{improved embed} of the next section which is essentially a stronger version of the straightforward result Lemma \ref{outer lemma} when $\Kh$ is a structured set obeying \eqref{structured dependence}.
 
 
\begin{proof} Create an $\eps=\delta^{3/2}$ covering $\bke$ of $K$. In order for covering elements to satisfy the embedding bound \eqref{rule} Theorem \ref{improved embed} requires $m\geq C\delta^{-2}C(K)$. Next we need to ensure that local deviation properties still hold. In particular given $\x',\y'\in \bke$ and $\x,\y\in K$ we still have $\ang(\x,\x')\leq \pi^{-1}\eps\leq \pi^{-1}\delta$ and for the rest we repeat the proof of Theorem \ref{prop second} which ends up yielding the following conditions (after an application of Lemma \ref{lemma basic})
\beq
c\eps \delta\sqrt{\log \frac{1}{\delta}} m+\sqrt{m\delta}\omega(K_\eps)\leq \frac{m\delta^2}{8},~m\geq \frac{64}{\delta}\log N_{\eps}\nn
\eeq
Both of these conditions trivially hold when we use the estimates \eqref{structured dependence}, namely $\omega(K_\eps)\leq\eps C(K)$ and $\log N_\eps \leq C(K)\log\eps^{-1}$.
\end{proof}

\section{Optimal embedding of structured sets}\label{bin chain}
As we mentioned previously, naive bounds for subspaces require a sample complexity of $\order{d\delta^{-2}\log\delta^{-1}}$ where $d$ is the subspace dimension. On the other hand, for linear embedding it is known that the optimal dependence is $\order{d/\delta^2}$. We will show that it is in fact possible to achieve optimal dependence via a more involved argument based on ``generic chaining'' strategy. The main result of this section is summarized in the following theorem.
\begin{theorem} \label{improved embed}Suppose $K$ satisfies the bounds \eqref{structured dependence} for all $\eps>0$. There exists constants $c,c_1,c_2>0$ and an $\eps=c\delta^{3/2}$ covering $\bke$ of $K$ such that if $m\geq c_1\delta^{-2}C(K)$, with probability $1-10\exp(-c_2\delta^2m)$ all $\x,\y\in \bke$ obey
\beq
|m^{-1}\ham{\sgn{\A\x}}{\sgn{\A\y}}-\ang(\x,\y)|\leq \delta .\nn
\eeq
\end{theorem}
\begin{proof}
Let $\Cc_i$ be a $\frac{1}{2^i}$ $\ell_2$-cover of $K$. From structured set assumption \eqref{structured dependence}, for some constant $C_0>0$ the cardinality of the covering satisfies
\beq
\log |\Cc_i|\leq i C(K).\nn
\eeq
Consider covers $\Cc_i$ for $1\leq i\leq N=\lceil\log_2 \frac{1}{\eps}\rceil$. This choice of $N$ ensures that $\Cc_N$ is an $\eps$ cover.
 
 Given points $\x=\x_N,\y=\y_N\subset \Cc_N$, find $\x_{1},\dots, \x_{N-1},\y_1,\dots,\y_{N-1}$ in the covers that are closest to $\x_N,\y_N$ respectively. For notational simplicity, define
\begin{align}
&4d(\x,\y)=\tn{\sgn{\A\x}-\sgn{\A\y}}^2,\nn\\
&4d(\x,\y,\z)=\li\sgn{\A\x}-\sgn{\A\y},\sgn{\A\y}-\sgn{\A\z}\ri,\nn\\
&4d(\x,\y,\z,\w)=\li\sgn{\A\x}-\sgn{\A\y},\sgn{\A\z}-\sgn{\A\w}\ri.\nn
\end{align}
Each of $d(\x,\y),d(\x,\y,\z),d(\x,\y,\z,\w)$ are sum of $m$ i.i.d.~random variables that take values $+1,-1,0$. For instance consider $d(\x,\y,\z,\w)$. In this case, the random variables are of the form
\beq
\ab=\li\sgn{\g^T\x}-\sgn{\g^T\y},\sgn{\g^T\z}-\sgn{\g^T\w}\ri,~\text{where}~\g\sim\Nn(0,\Iden_n),\nn
\eeq
which is either $-1,0,1$. Furthermore, it is $0$ as soon as either $\x,\y$ or $\z,\w$ induces the same sign which means $\Pro(\ab\neq 0)\leq \min\{\ang(\x,\y),\ang(\z,\w)\}$.

Given points $\{\x_N,\dots,\x_1,\y_1,\dots,\y_N\}$ we have that
\begin{align}
d(\x_i,\y_i)=&d(\x_i,\x_{i-1})+d(\y_i,\y_{i-1})+2d(\x_i,\x_{i-1},\y_{i-1})+2d(\x_{i-1},\y_{i-1},\y_{i})+2d(\x_i,\x_{i-1},\y_{i-1},\y_i),\label{first line}\\
&+d(\x_{i-1},\y_{i-1}).\nn
\end{align}
The term $d(\x_{i-1},\y_{i-1})$ on the second line will be used for recursion. We will show that each of the remaining terms (first line) concentrate around their expectations. Since the argument is identical, to prevent repetitions, we will focus on $d(\x_i,\x_{i-1},\y_{i-1},\y_i)$.

Recall that $d(\x_i,\x_{i-1},\y_{i-1},\y_i)$ is sum of $m$ i.i.d.~random variables $\ab_i$ that takes values $\{-1,0,1\}$ where $\ab_i$ satisfies
\beq
\Pro(|\ab_i|=1)\leq \min\{\ang(\x_i,\x_{i-1}),\ang(\y_i,\y_{i-1})\}\leq \frac{2}{2^i}:=2\delta_i\nn
\eeq
as $\ang(\x,\y)\leq \tn{\x-\y}/2$. Assuming $\eps_i\leq {\delta_i}{}$ (will be verified at \eqref{eps verify}), for a particular quadruple $(\x_i,\x_{i-1},\y_{i-1},\y_i)$, applying Lemma \ref{lem triple} for $i\geq 4$ (which ensures $\mu/2\leq \Pro(|\ab_i|=1)\leq 1/6$), we have that
\beq
\Pro(|d(\x_i,\x_{i-1},\y_{i-1},\y_i)-\E[d(\x_i,\x_{i-1},\y_{i-1},\y_i)]|\geq \eps_i m)\leq 2\exp(-\frac{\eps_i^2 m}{4\delta_i}).\label{bound app}
\eeq
Pick $\eps_i=\sqrt{i}2^{-i/2}\delta$. Observe that, $\frac{\eps_i}{\delta_i}$ can be bounded as
\beq
\frac{\eps_i}{\delta_i}=\sqrt{i}\frac{2^{-i/2}}{2^{-i}}\delta=\sqrt{i}2^{i/2}\delta\leq \sqrt{N}2^{N/2}\delta\leq \sqrt{N}2^{N/2}2^{-2(N-1)/3}<1/4\label{eps verify}
\eeq
for $\eps$ (or $\delta$) sufficiently small (which makes $N$ large) where we used the fact that $2^{-(N-1)}\geq \eps=\delta^{3/2}$. Consequently the bound \eqref{bound app} is applicable.

This choice of $\eps_i$ yields a failure probability of $2\exp(-\frac{\eps_i^2 m}{4\delta_i})=2\exp(-i\delta^2m/4)$. Union bounding \eqref{bound app} over all quadruples we find that the probability of success for all $(\x_i,\x_{i-1},\y_{i-1},\y_i)$ is at least
\beq
1-2\exp(-i\delta^2m/4)\exp(4iC(K))\leq 1-2\exp(-i\delta^2m/8)\nn
\eeq
under initial assumptions. The deviation of the other terms in \eqref{first line} can be bounded with the identical argument. Define $\eta_i$ to be
\beq
\eta_i=d(\x_i,\x_{i-1})+d(\y_i,\y_{i-1})+2d(\x_i,\x_{i-1},\y_{i-1})+2d(\x_{i-1},\y_{i-1},\y_{i})+2d(\x_i,\x_{i-1},\y_{i-1},\y_i).\nn
\eeq
So far we showed that for all $\x_i,\y_i,\x_{i-1},\y_{i-1}$, with probability $1-10\exp(-i\delta^2m/8)$
\beq
|\eta_i-\E[\eta_i]|\leq 5\eps_i m=5\sqrt{i}2^{-i/2}\delta m.\nn
\eeq
Since $d(\x_i,\y_i)-d(\x_{i-1},\y_{i-1})=\eta_i$ applying a union bound over $4\leq i\leq N$ we find that
\beq
d(\x_N,\y_N)=\sum_{i=4}^N\eta_i+d(\x_3,\y_3)
\eeq
We treat $d(\x_3,\y_3)$ specifically as Lemma \ref{lem triple} may not apply. In this case, the cardinality of the cover is small in particular $\log |\Cc_3|\leq k(\log C_0+3)+\log L$ hence, we simply use Lemma \ref{outer lemma} to conclude for all $\x_3,\y_3$
\beq
|d(\x_3,\y_3)-\E[d(\x_3,\y_3)]|\leq \delta\nn
\eeq
with probability $1-\exp(-\delta^2m)$. Merging our estimates, and using $\E[d(\x_N,\y_N)]=\sum_{i=4}^N \E[\eta_i]+\E[d(\x_3,\y_3)]$, we find
\beq
|d(\x_N,\y_N)-\E[d(\x_N,\y_N)]|\leq \delta+\sum_{i=4}^N5\sqrt{i}2^{-i/2}\delta\leq c'\delta\nn
\eeq
for an absolute constant $c'>0$ with probability $1-10\sum_{i=1}^N\exp(-i\delta^2m/8)\leq 1-10\frac{\exp(-\delta^2m/8)}{1-\exp(-\delta^2 m/8)}$. Clearly, our initial assumption allows us to pick $\delta^2 m\geq 16$ so that $1-\exp(-\delta^2 m/8)\geq 0.5$ which makes the probability of success $1-20\exp(-\delta^2m/8)$. To obtain the advertised result simply use the change of variable $c'\delta\rightarrow \delta$.
\end{proof}
The remarkable property of this theorem is the fact that we can fix the sample complexity to $\delta^{-2}C(K)$ while allowing the $\eps$-net to get tighter as a function of the distortion $\delta$. We should emphasize that $\eps\sim \delta^{3/2}$ dependence can be further improved to $\eps\sim\delta^{2-\alpha}$ for any $\alpha>0$. We only need to ensure that \eqref{eps verify} is satisfied. The next result is a helper lemma for the proof of Theorem \ref{improved embed}.
\begin{lemma}\label{lem triple} Let $\{\x_i\}_{i=1}^m$ be i.i.d.~random variables taking values in $\{-1,0,1\}$. Suppose $\max\{\Pro(\x_1=-1)=p_-,\Pro(\x_1=1)\}=p_+\}\leq \mu/2$ for some $0\leq\mu\leq 1/3$. Then, whenever $\eps\leq \frac{\mu}{2}$
\beq
\Pro(|m^{-1}\sum_{i=1}^m\x_i-\E[m^{-1}\sum_{i=1}^m\x_i]|\leq \eps )\geq 1-2\exp(-\frac{\eps^2m}{4\mu}).\nn
\eeq

\end{lemma}
\begin{proof} Let us first estimate the number of nonzero components in $\{\x_i\}$. Using a standard Chernoff bound (e.g.~Lemma \ref{modified sheriff}), for $\eps<\mu/2$, we have that
\beq
\Pro(|m^{-1}\sum_{i=1}^m|\x_i|-(p_++p_-)|\leq \eps )\geq 1-\exp(-\frac{\eps^2m}{4\mu}).\label{union 1}
\eeq
Conditioned on $\sum_{i=1}^m|\x_i|=c\leq (\mu+\eps)m\leq 1.5\mu m$, the $c$ nonzero elements in $\x_i$ are $+1,-1$ with the normalized probabilities $\frac{p_+}{p_++p_-},\frac{p_-}{p_++p_-}$. Denoting these variables by $\{b_i\}_{i=1}^c$, and applying another Chernoff bound, we have that
\beq
\Pro(|\sum_{i=1}^c b_i-\E[\sum_{i=1}^c b_i]|\leq \eps_2 c)\geq 1-\exp(-2\eps_2^2 c).\nn
\eeq
Picking $\eps_2=\frac{\eps m}{c}$, we obtain the bound 
\beq
\Pro(|\sum_{i=1}^c b_i-\E[\sum_{i=1}^c b_i]|\leq \eps m)\geq 1-\exp(-2\eps^2 m^2/c)\label{union 2}
\eeq
which shows that $|\sum_{i=1}^c b_i-\frac{p_+-p_-}{p_++p_-}c|\leq \eps m$ with probability $1-\exp(-2\eps^2 m^2/c)$. Since $c\leq 2\mu m$, $1-\exp(-2\eps^2 m^2/c)\geq 1-\exp(-\eps^2 m/\mu)$. Finally observe that
\beq
|\frac{p_+-p_-}{p_++p_-}c-(p_+-p_-)m|\leq \frac{|p_+-p_-|}{p_++p_-}\eps m\leq\eps m\implies |\sum_{i=1}^c b_i-(p_+-p_-)m|\leq 2\eps m.\nn
\eeq
To conclude recall that $\sum_{i=1}^cb_i=\sum_{i=1}^m \x_i$ and $(p_+-p_-)m=\E[\sum_{i=1}^m \x_i]$, then apply a union bound combining \eqref{union 1} and \eqref{union 2}.
\end{proof}

\section{Local properties of binary embedding}\label{sec local}
For certain applications such as locality sensitive hashing, we are more interested with the local behavior of embedding, i.e.~what happens to the points that are around each other. In this case, instead of preserving the distance, we can ask for $\sgn{\A\x},\sgn{\A\y}$ to be close if and only if $\x,\y$ are close.

The next theorem is a restatement of the second statement of Theorem \ref{thm main 3} and summarizes our result on local embedding.
\begin{theorem} Given $0<\delta<1$, set $\eps=c\delta/\sqrt{\log\delta^{-1}}$. There exists $c,c_1,c_2>0$ such that if
\beq
m\geq c_1\max\{\frac{1}{\delta}\log N_\eps,\delta^{-3}\omega^2(K_\eps)\}\nn.
\eeq
with probability $1-\exp(-c_2\delta m)$, the following statements hold.
\begin{itemize}
\item For all $\x,\y\in K$ satisfying $\ang(\x,\y)\leq \eps$, we have $m^{-1}\|\sgn{\A\x},\sgn{\A\y}\|_H\leq \delta$.
\item For all $\x,\y\in K$ satisfying $m^{-1}\|\sgn{\A\x},\sgn{\A\y}\|_H\leq \delta/32$, we have $\ang(\x,\y)\leq \delta$.
\end{itemize}
\end{theorem}
\begin{proof} First statement is already proved by Theorem \ref{prop second}. For the second statement, we shall follow a similar argument. Suppose that for some pair $\x,\y\in K$ obeying $\ang(\x,\y)>\delta$ we have $\|\sgn{\A\x},\sgn{\A\y}\|_H\leq  \delta/32$. 
 
Consider an $\eps$ covering $\bke$ of $K$ where $\eps=c\delta/\sqrt{\log\delta^{-1}}$ for a sufficiently small $c>0$. Let $\x',\y'$ be the elements of the cover obeying $\tn{\x'-\x},\tn{\y'-\y}\leq \eps$ which also ensures the angular distance to be at most $\pi^{-1}\eps$. Applying Theorem \ref{prop second} under the stated conditions we can ensure that
\beq
m^{-1}\|\A\x,\A\x'\|_H,m^{-1}\|\A\y,\A\y'\|_H\leq \delta/20.\nn
\eeq
Next, since $\eps$ can be made arbitrarily smaller than $ \delta$ we can guarantee that $\ang(\x',\y')\geq \delta/2$. We shall apply a Chernoff bound over the elements of the cover to ensure that $\|\A\x'-\A\y'\|_H$ is significant for all $\x',\y'\in \bke$. The following version of Chernoff gives the desired bound.
\begin{lemma} Suppose $\{a_i\}_{i=1}^m$ are i.i.d.~Bernoulli random variables satisfying $\E[a_i]\geq \delta/2$. Then, with probability $1-\exp(-\delta m /16)$ we have that $\sum_{i=1}^m a_i\geq  \delta m/4$.
\end{lemma}
Applying this lemma to all pairs of the cover, whenever $m\geq 128\delta^{-1}\log N_\eps$ we find that $\|\A\x'-\A\y'\|_H\geq  \delta m/4$. 

We can now use the triangle inequality to achieve
\begin{align}
|m^{-1}\|\A\x,\A\y\|_H-\ang(\x,\y)|&\geq |m^{-1}\|\A\x',\A\y'\|_H-\ang(\x',\y')|\label{contrib 1}\\
&-[|\ang(\x,\y)-\ang(\x',\y)|+|\ang(\x',\y)-\ang(\x',\y')|]\label{contrib 2}\\
&-[|m^{-1}\|\A\x,\A\y\|_H-m^{-1}\|\A\x',\A\y\|_H|+|m^{-1}\|\A\x',\A\y\|_H-m^{-1}\|\A\x',\A\y'\|_H|]\label{contrib 3}
\end{align}
\eqref{contrib 1} is at least $\delta /4$, \eqref{contrib 2} is at most $c_0\eps$ and \eqref{contrib 3} is at most $\delta/10$ which ensures that $|m^{-1}\|\A\x,\A\y\|_H-\ang(\x,\y)|\geq \delta/8$ by picking $c$ small enough. This contradicts with the initial assumption. To obtain this contradiction, we required $m$ to be $m\geq c_1\max\{\delta^{-1}\log N_\eps,\delta^{-3}\omega^2(K_\eps)\}$ for some constant $c_1>0$.
\end{proof}

\section{Sketching for binary embedding}\label{sec:sketch}
In this section, we discuss preprocessing the binary embedding procedure with a linear embedding. Our goal is to achieve an initial dimensionality reduction that preserves the $\ell_2$ distances of $K$ with a linear map and then using binary embedding to achieve reasonable guarantees. In particular we will prove that this scheme works almost as well as the Gaussian binary embedding we discussed so far.

For the sake of this section, $\B\in\R^{m\times \ml}$ denotes the binary embedding matrix and $\F\in\R^{\ml\times n}$ denotes the preprocessing matrix that provides an initial sketch of the data. The overall sketched binary embedding is given by $\x\rightarrow \sgn{\A\x}$ where $\A=\B\F\in\R^{m\times n}$. We now provide a brief background on linear embedding.

\subsection{Background on Linear Embedding}
For a mapping to be a linear embedding, we require it to preserve distances and lengths of the sets.
\begin{definition} [$\delta$-linear embedding]\label{lin embed} Given $\delta\in (0,1)$, $\F$ is a $\delta$-embedding of the set $\Cc$ if for all $\x,\y\in\Cc$, we have that
\beq
|\tn{\F\x-\F\y}-\tn{\x-\y}|\leq \delta,~|\tn{\A\x}-\tn{\x}|\leq \delta.
\eeq
\end{definition}
Observe that if $\Cc$ is a subset of the unit sphere, this definition preserves the length of the vectors multiplicatively i.e.~obeys $|\|\A\x\|-\|\x\||\leq \delta \|\x\|$. We should point out that more traditional embedding results ask for 
\beq
|\tn{\F\x-\F\y}^2-\tn{\x-\y}^2|\leq \delta.\nn
\eeq
This condition can be weaker than what we state as it allows for $|\tn{\F\x-\F\y}-\tn{\x-\y}|\sim \sqrt{\delta}$ for small values of $\tn{\x-\y}$. However, Gaussian matrices allow for $\delta$-linear embedding with optimal dependencies. The reader is referred to Lemma $6.8$ of \cite{oymak2015sharp} for a proof.
\begin{theorem}\label{gauss embed} Suppose $\F\in\R^{\ml\times n}$ is a standard Gaussian matrix normalized by $\sqrt{\ml}$. For any set $\Cc\in \Bc^{n-1}$ whenever $\sqrt{\ml}\geq \omega(\Cc)+\eta+1$, with probability $1-\exp(-\eta^2/8)$ we have that
\beq
\sup_{\x\in \Cc}|\tn{\F\x}-\tn{\x}|\leq \ml^{-1/2}(\omega(\Cc)+\eta).\nn
\eeq
\end{theorem}
To achieve a $\delta$-linear embedding we can apply this to the sets $K-K$ and $K$ by setting $\ml=\order{\delta^{-2}(\omega(\Cc)+\eta)^2}$. 

Similar results can be obtained for other matrix ensembles including Fast Johnson-Lindenstrauss Transform \cite{oymak2015isometric}. A Fast JL transform is a random matrix $\F=\Sb\Db\Rb$ where $\Sb\in\R^{\ml\times n}$ randomly subsamples $m$ rows of a matrix, $\Db\in\R^{n\times n}$ is the normalized Hadamard transform, and $\Rb$ is a diagonal matrix with independent Rademacher entries. The recent results of \cite{oymak2015isometric} shows that linear embedding of arbitrary sets via FJLT is possible. The following corollary follows from their result by using the fact that $|\tn{\A\x}^2-\tn{\x}^2|\leq \delta^2\implies |\tn{\A\x}-\tn{\x}|\leq\delta$.
\begin{corollary} \label{cor FJLT}Suppose $K\subset\Sc^{n-1}$. Suppose $\ml\geq c(1+\eta)^2\delta^{-4}(\log n)^4\omega^2(K)$ and $\F\in\R^{\ml\times n}$ is an FJLT. Then, with probability $1-\exp(-\eta)$, $\F$ is a $\delta$-linear embedding for $K$.
\end{corollary}
Unlike Theorem \ref{gauss embed}, this corollary yields $d\sim\delta^{-4}$ instead of $\delta^{-2}$. It would be of interest to improve such distortion bounds. 
\subsection{Results on sketched binary embedding}
With these technical tools, we are in a position to state our results on binary embedding.

\begin{theorem} \label{propo bin lin}Suppose $\F\in\R^{\ml\times m}$ is a $\delta$-linear embedding for the set $K\bigcup -K$. Then
\begin{itemize}
\item $\B\F$ is a $c\delta$-binary embedding of $K$ if $\B$ is a $\delta$-binary embedding for $\Sc^{\ml-1}$.
\item Suppose $\Kh$ is union of $L$ $d$-dimensional subspaces, then $\B\F$ is a $c\delta$-binary embedding with probability $\alpha$ if $\B$ is a $\delta$-binary embedding for union of $L$ $d$-dimensional subspaces with the same probability.
\end{itemize}
\end{theorem}
\begin{proof} Given $\x,\y\in K$ define $\x_F=\frac{\F\x}{\|\F\x\|}$, $\y_F=\frac{\F\y}{\|\F\y\|}$ where $\A\x=\B\F\x$. Since $\F$ is a $\delta$-linear embedding for $K$, for all $\x,\y$  the following statements hold.
\begin{itemize}
\item $\max\{|\|\F\x\|-\|\x\||,|\|\F\y\|-\|\y\||,|\|\F\x-\F\y\|-\|\x-\y\||\}\leq \delta$.
\item $|\|\y_F-\x_F\|-\|\F\x-\F\y\||\leq \|\F\x-\x_F\|+ \|\F\y-\y_F\|\leq 2\delta$.
\end{itemize}
These imply $|\|\y_F-\x_F\|-\|\x-\y\||\leq 3\delta$ which in turn implies $|\ang(\x_F,\y_F)-\ang(\x,\y)|\leq c\delta$ using standard arguments in particular Lipschitzness of the geodesic distance between $0$ and $\frac{\pi}{2}$. We may assume the angle to be between $0$ and $\pi/2$ as the set $K\bigcup-K$ is symmetric and if $\ang(\x,\y)>\pi/2$ we may consider $\ang(\x,-\y)<\pi/2$.

If $\B$ is a $\delta$ binary embedding it implies that $|\ang(\x_F,\y_F)-m^{-1}\|\sgn{\A\x},\sgn{\A\y}\|_H|\leq \delta$. In turn, this yields $|\ang(\x,\y)-m^{-1}\|\sgn{\A\x},\sgn{\A\y}\|_H|\leq (c+1)\delta$ and overall map is $(c+1)\delta$-binary embedding.

For the second statement observe that if $\Kh$ is a union of subspaces so is $\text{cone}(\F K)$ and we are given that $\B$ is a $\delta$ binary embedding for union of subspaces with $\alpha$ probability. This implies that $|\ang(\x_F,\y_F)-\|\sgn{\B\x_F},\sgn{\B\y_F}\|_H|\leq \delta$ with probability $\alpha$ and we again conclude with $|\ang(\x,\y)-m^{-1}\|\sgn{\A\x},$ $\sgn{\A\y}\|_H|\leq (c+1)\delta$.
\end{proof}

Our next result obtains a sample complexity bound for sketched binary embedding.
\begin{theorem} There exists constants $c,C>0$ such that if $\B\sim \R^{m\times \ml}$ is a standard Gaussian matrix satisfying $m\geq c\delta^{-2}\ml$ and
\begin{itemize}
\item if $\F$ is a standard Gaussian matrix normalized by $\sqrt{\ml}$ where $\ml>c\delta^{-2}\omega^2(K)$, with probability $1-\exp(-C\delta^2 m)-\exp(-C\delta^2 \ml)$ $\x\rightarrow\sgn{\A\x}$ for $\A=\B\F$ is a $\delta$-binary embedding of $K$.
\item if $\F$ is a Fast Johnson-Lindenstrauss Transform where $d>c(1+\eta)^2\delta^{-4}(\log n)^4\omega^2(K)$, with probability $1-\exp(-\eta)-\exp(-C\delta^2m)$ $\x\rightarrow\sgn{\A\x}$ for $\A=\B\F$ is a $\delta$-binary embedding of $K$.
\end{itemize}
\end{theorem}
Observe that for Gaussian embedding distortion dependence is $\delta^{-4}\omega^2(K)$ which is better than what can be obtained via Theorem \ref{thm main 1}. This is due to the fact that here we made use of the improved embedding result for subspaces. On the other hand distortion dependence for FJLT is $\delta^{-6}$ which we believe to be an artifact of Corollary \ref{cor FJLT}. We remark that better dependencies can be obtained for subspaces.

\begin{proof} The proof makes use of the fact that $\B$ is a $\delta$ binary embedding for $\R^{\ml}$ with the desired probability. Consequently, following Theorem \ref{propo bin lin}, we simply need to ensure $\F$ is a $\delta$ linear embedding. When $\F$ is Gaussian this follows from Theorem \ref{gauss embed}. When $\F$ is FJLT, it follows from the bound Corollary \ref{cor FJLT} namely $\ml\geq c(1+\eta)^2\delta^{-4}(\log n)^4\omega^2(K)$.
\end{proof}

\subsection{Computational aspects}\label{sec:compute}
Consider the problem of binary embedding a $d$-dimensional subspace $\Kh$ with $\delta$ distortion. Using a standard Gaussian matrix as our embedding strategy requires $\order{\delta^{-2}nd}$ operation per embedding. This follows from the $\order{mn}$ operation complexity of matrix multiplication where $m=\order{\delta^{-2}d}$.

For FJLT sketched binary embedding, setting $\ml=\order{\delta^{-4}(\log n)^4d}$ and $m=\order{\delta^{-2}d}$ we find that each vector can be embedded in $\order{n\log n+\ml m}\approx\order{n+\delta^{-6}d^2}$ up to logarithmic factors. Consequently, in the regime $d=\order{\sqrt{n}}$, sketched embedding strategy is significantly more efficient and embedding can be done in near linear time. The similar pattern arises when we wish to embed an arbitrary set $K$. The main difference will be the distortion dependence in the computation bounds. We omit this discussion to prevent repetition.

These tradeoffs are in similar flavor to the recent work by Yi et al. \cite{yi2015binary} that apply to the embedding of finite sets. Here we show that similar tradeoffs can be extended to the case where $K$ is arbitrary.

We shall remark that a faster binary embedding procedure is possible in practice via FJLT. In particular pick $\A=\Sb\Db\Gb\Db^*\Rb$ where $\Sb$ is the subsampling operator, $\Db$ is the Hadamard matrix and $\Gb$ and $\Rb$ are diagonal matrices with independent standard normal and Rademacher nonzero entries respectively \cite{yu2014circulant,le2013fastfood}. In this case, it is trivial to verify that for a given pair of $\x,\y$ we have the identity $\ang(\x,\y)=m^{-1}\E\|\A\x,\A\y\|_H$. While expectation trivially holds analysis of this map is more challenging and there is no significant theoretical result to the best of our knowledge. This map is beneficial since using FJLT followed by dense Gaussian embedding results in superlinear computation in $n$ as soon as $d\geq \order{\sqrt{n}}$. On the other hand, this version of Fast Binary Embedding has near-linear embedding time due to diagonal multiplications. Consequently, it would be very interesting to have a guarantee for this procedure when subspace $\Kh$ is in the nontrivial regime $d= \Omega({\sqrt{n}})$. Investigation of this map for both discrete and continuous embedding remains as an open future direction.
\section{Numerical experiments}\label{sec:numerics}
\begin{figure}
\hspace{-15pt}\centering
\begin{subfigure}[b]{0.475\linewidth} \centering\includegraphics[width=1.15\linewidth]{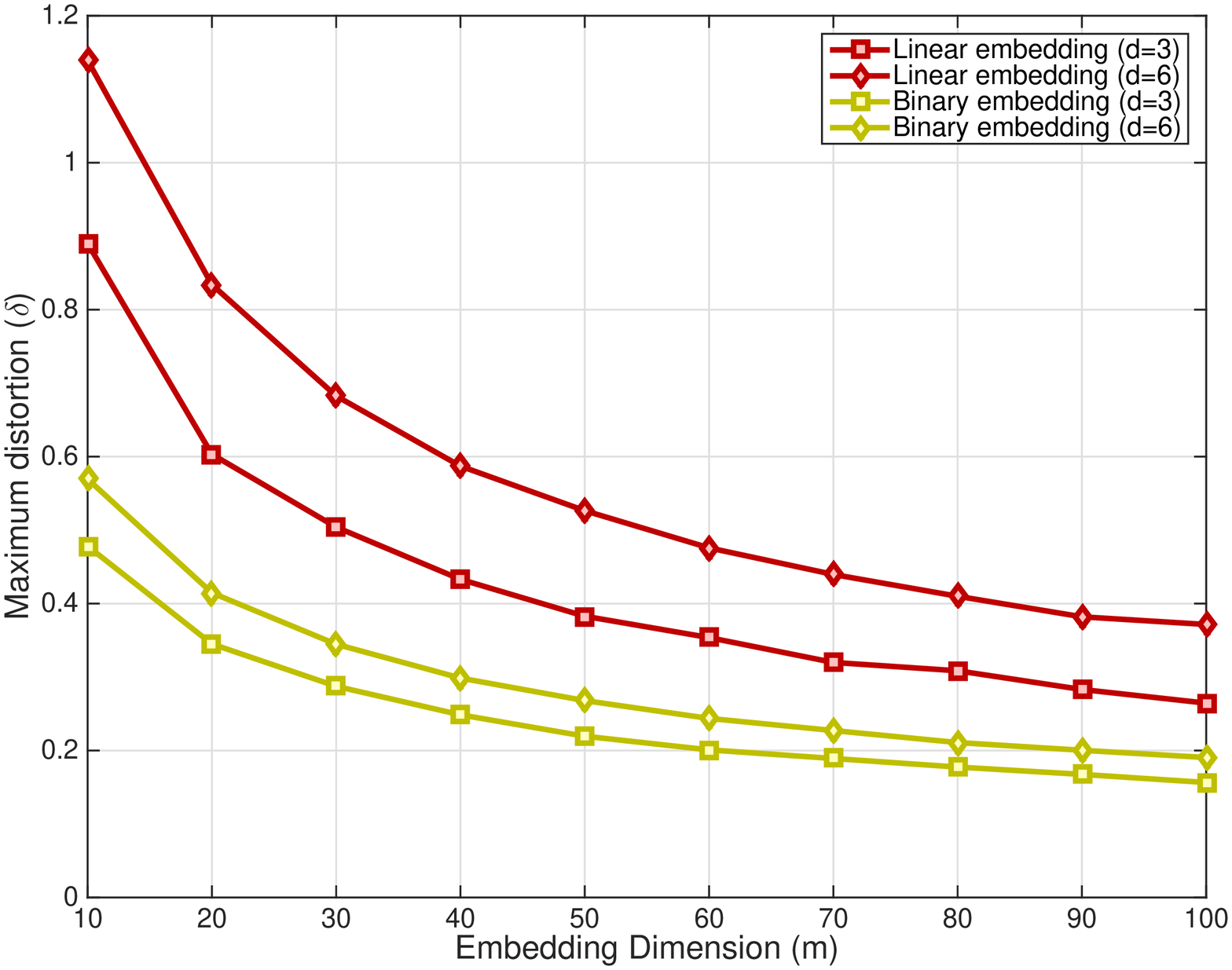} \caption{Binary vs linear embedding} \label{fig:unnnormed} 
\end{subfigure} ~~
\begin{subfigure}[b]{0.475\linewidth} \centering{\includegraphics[width=1.15\linewidth]{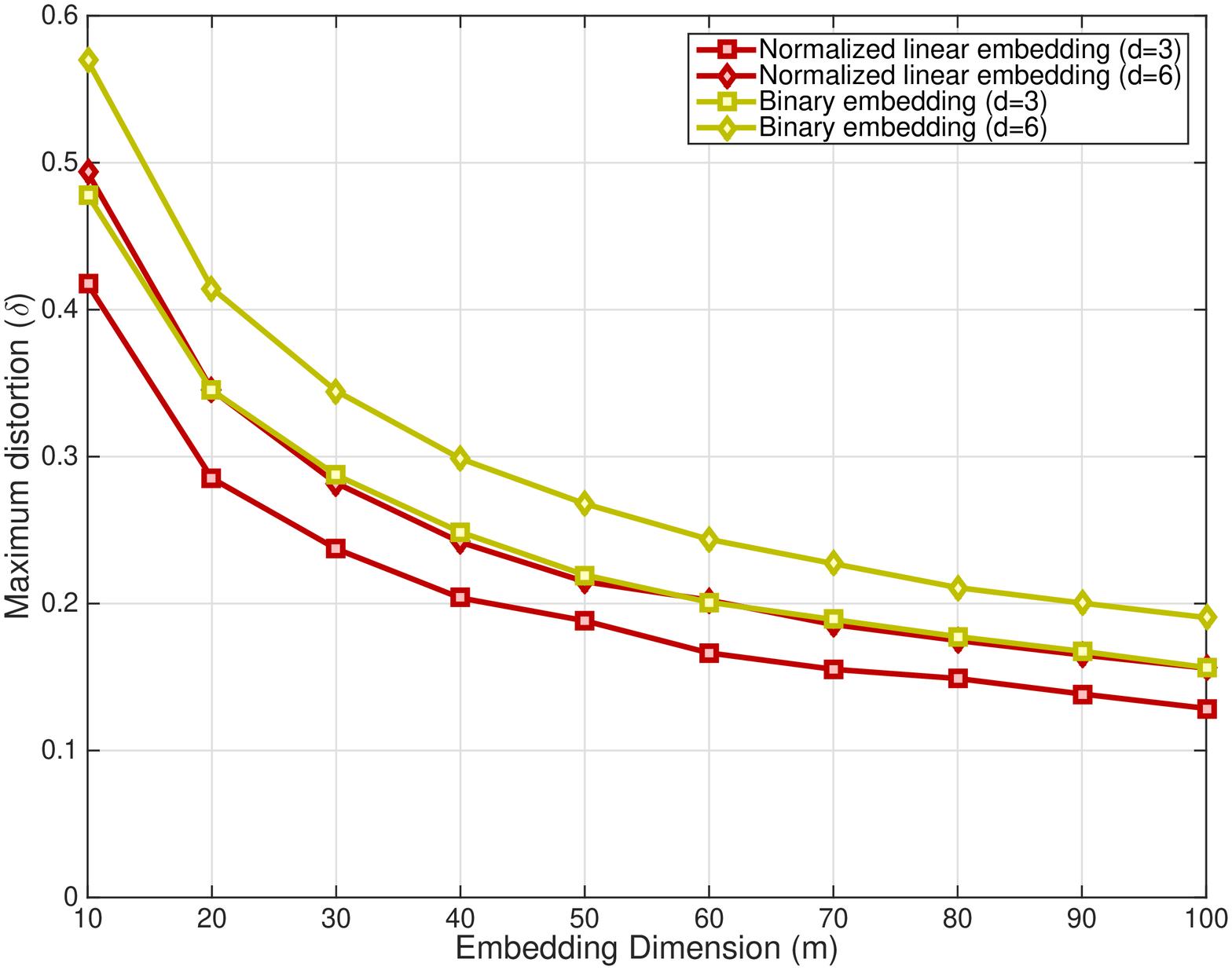}} \caption{Binary vs normalized linear embedding.} \label{fig:norm} \end{subfigure} 
\caption{(a) Comparison of binary and linear embedding for a Gaussian map as a function of subspace dimension and sample complexity. (b) We keep the same setup however use a normalized distortion function. With this change, linear and binary embedding becomes more comparable.}\label{fig 1}
\end{figure}
Next, we shall list our numerical observations on binary embedding. A computational difficulty of binary embedding for continuous sets is the fact that it is nontrivial to obtain distortion bounds. In particular given set $K$ and map $\A$ we would like to quantify the maximum distortion given by
\beq
\delta_{bin}=\sup_{\x,\y\in K} |m^{-1}\|\A\x,\A\y\|_H-\ang(\x,\y)|.\nn
\eeq 
To the best of our knowledge there is no guaranteed way of finding this supremum. For instance, for linear embedding we are interested in the bound
\beq
\delta_{lin}=\sup_{\x,\y\in K} |m^{-1}\tn{\A\x,\A\y}-\tn{\x,\y}|\nn
\eeq
which can be obtained by calculating minimum and maximum singular values of $\A$ restricted to $K$. When $K$ is a subspace, this can be done efficiently by studying the matrix obtained by projecting rows of $\A$ onto $K$.

To characterize the impact of set size on distortion, we sample $200$ points from a $d$-dimensional subspace for $d=3$ and $d=6$ where $n=128$. Clearly sampling finite number of points is not sufficient to capture the whole space however it is a good enough proxy for illustrating the impact of subspace dimension. We additionally vary the number of samples $m$ between $0$ and $100$. Figure \ref{fig:unnnormed} contrasts linear and binary embedding schemes where $\sqrt{m}\A$ is a standard Gaussian matrix. We confirm that sampling the points from larger subspace indeed results in a larger distortion for both cases. Interestingly we observe that distortion for binary embedding is smaller than linear. This is essentially due to the fact that cost functions are not comparable rather than their actual performance. Clearly linear embedding stores more information about the signal so we expect it to be more beneficial.

For a better comparison, we normalize the linear distortion function with respect to binary distortion so that if $\tn{\A\x,\A\y}=\|\A\x,\A\y\|_H=0$ we have $\delta_{bin}$ is same as normalized distortion $\delta_{n-lin}$. This corresponds to the function
\beq
\delta_{n-lin}=\frac{\ang(\x,\y)}{\tn{\x,\y}}\sup_{\x,\y\in K} |m^{-1}\tn{\A\x,\A\y}-\tn{\x,\y}|.\nn
\eeq
Figure \ref{fig:norm} shows the comparison of $\delta_{bin}$ and $\delta_{n-lin}$. We observe that linear embedding results in lower distortion but their behavior is highly similar. Figure \ref{fig 1} shows that linear and binary embedding performs on par and linear embedding does not provide a significant advantage. This is consistent with the main message of this work.


In Figure \ref{fig:fjlt} we compare Gaussian embedding with fast binary embedding given by $\A=\Sb\Db\Gb\Db^*\Rb$ as described in Section \ref{sec:compute}. We observe that Gaussian yields slightly better bounds however both techniques perform on par in all regimes. This further motivates theoretical understanding of fast binary embedding which significantly lags behind linear embedding. Sparse matrices are another strong alternative for fast multiplication and efficient embedding \cite{dasgupta2010sparse}. Figure \ref{fig:sparse} contrasts Gaussian embedding with sparse Gaussian where the entries are $0$ with probability $2/3$. Remarkably, the distortion dependence perfectly matches. Following from this example, it would be interesting to study the class of matrices that has the same empirical behavior as a Gaussian. For linear embedding this question has been studied extensively \cite{Universal,oymak2015universality} and it remains as an open problem whether results of similar flavor would hold for binary embedding.

\begin{figure}
\hspace{-15pt}\centering
\begin{subfigure}[b]{0.475\linewidth} \centering\includegraphics[width=1.15\linewidth]{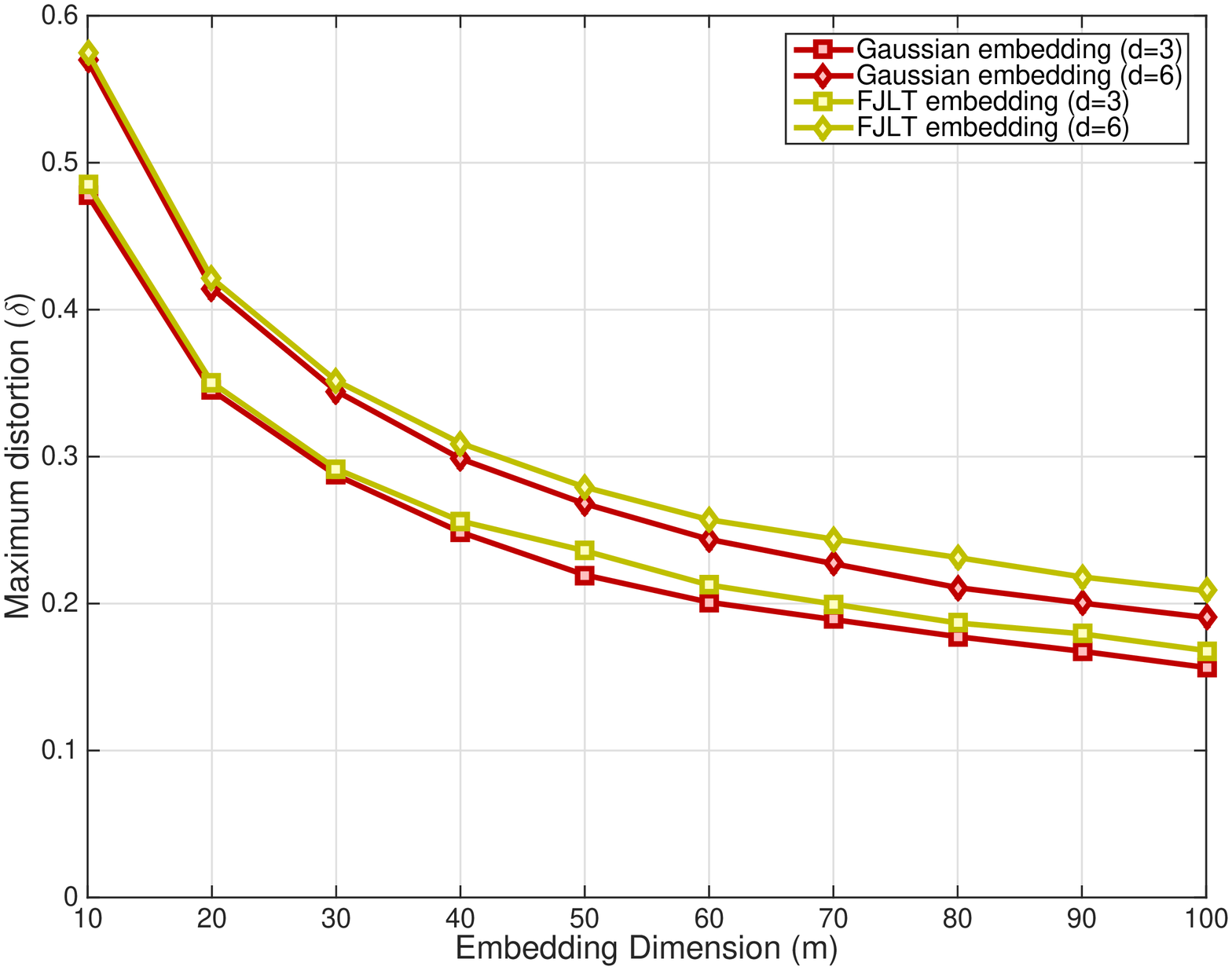} \caption{Gaussian vs FJLT for binary embedding.} \label{fig:fjlt} 
\end{subfigure} ~~
\begin{subfigure}[b]{0.475\linewidth} \centering{\includegraphics[width=1.15\linewidth]{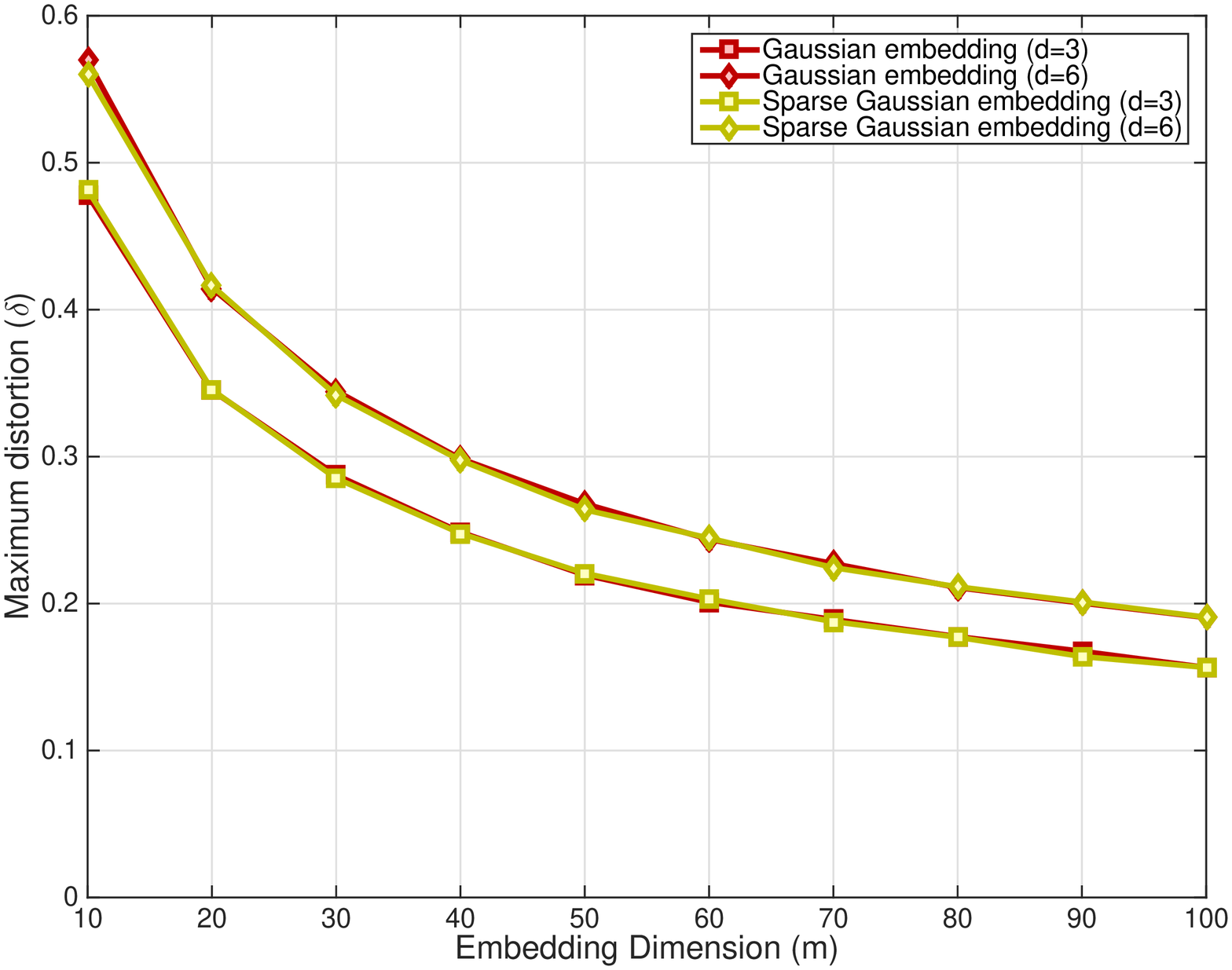}} \caption{Gaussian vs sparse Gaussian for binary embedding.} \label{fig:sparse} \end{subfigure} 
\end{figure}

\section*{Acknowledgments}
The authors would like to thank Mahdi Soltanolkotabi for helpful conversations. BR is generously supported by ONR awards N00014-11-1-0723 and N00014-13-1-0129, NSF awards CCF-1148243 and CCF-1217058, AFOSR award FA9550-13-1-0138, and a Sloan Research Fellowship. SO was generously supported by the Simons Institute for the Theory of Computing and NSF award CCF-1217058.

\bibliography{Bibfiles}
\bibliographystyle{plain}

\appendix
\section{Supplementary results}
\begin{lemma} \label{modified sheriff}Suppose $\{\ab_i\}_{i=1}^n$ are Bernoulli-$q$ random variables where $0\leq q\leq p$. Suppose $p<1/3$. For any $\eps\leq \frac{p}{2}$, we have that
\beq
\Pro(|\frac{\sum_i \ab_i}{n}-q|>\eps)\leq \exp(-\frac{\eps^2n}{4p}).\nn
\eeq
\end{lemma}
\begin{proof} For any $\eps>0$, standard Chenoff bound yields that
\beq
\Pro(|\frac{\sum_i \ab_i}{n}-q|>\eps)\leq \exp(-nD(q+\eps||q))+\exp(-nD(q-\eps||q))\nn
\eeq
where $D(\cdot)$ is the KL-divergence. Applying Lemma \ref{lemma decrease}, we have that
\beq
\exp(-D(q+\eps||q))\leq \exp(-D(p+\eps||p)),~\exp(-D(q-\eps||q))\leq \exp(-D(p-\eps||p)).\nn
\eeq
Finally, when $\frac{|\eps|}{p}\leq \frac{1}{2}$, we make use of the fact that $D(p+\eps||p)\geq \frac{\eps^2n}{4p}$.
\end{proof}
\begin{lemma} \label{lemma decrease}Given nonnegative numbers $p_1,p_2,q_1,q_2$, suppose $0.5>p_1>p_2$, $d=p_1-q_1>0$, and $p_1-p_2=q_1-q_2$. Then
\beq
D(q_1||q_2)\geq D(p_1||p_2),~D(q_2||q_1)\geq D(p_2||p_1).\nn
\eeq
\end{lemma}
\begin{proof} For $p_2<p_1\leq 1/2$, using the definition of KL divergence, we have that
\beq
D(p_1||p_2)=\int_{p_2}^{p_1}\frac{p_1-x}{x(1-x)}dx=\int_{q_2}^{q_1}\frac{p_1-(x+d)}{(x+d)(1-(x+d))}dx=\int_{q_2}^{q_1}\frac{q_1-x}{(x+d)(1-(x+d))}dx.\nn
\eeq
Now, observe that for $x+d<1/2$
\beq
(x+d)(1-(x+d))\geq x(1-x)\implies \frac{q_1-x}{(x+d)(1-(x+d))}\leq \frac{q_1-x}{x(1-x)},\nn
\eeq
which yields $D(p_1||p_2)\leq D(q_1||q_2)$. The same argument can be repeated for $D(p_2||p_1)$. We have that
\beq
D(p_2||p_1)=\int_{1-p_1}^{1-p_2}\frac{1-p_2-x}{x(1-x)}dx=\int_{1-q_1}^{1-q_2}\frac{1-p_2-(x-d)}{(x-d)(1-(x-d))}dx=\int_{1-q_1}^{1-q_2}\frac{1-q_2-x}{(x-d)(1-(x-d))}dx.\nn
\eeq
This time, we make use of the inequality $(x-d)(1-(x-d))\geq x(1-x)$ for $x-d\geq 1/2$ to conclude.
\end{proof}
\begin{lemma} \label{q tail}Let $Q(a)=\Pro(|g|\geq  a)$ for $g\sim\Nn(0,1)$. There exists constants $c_1,c_2>0$ such that for $\delta>c_1$ we have that
\beq
\int_{Q^{-1}(\delta)}^\infty t\sqrt{2/\pi}\exp(-t^2/2)dt\leq c_2\delta \sqrt{\log\delta^{-1}}.\nn
\eeq
\end{lemma}
\begin{proof} Set $\gamma=Q^{-1}(\delta)$. Observe that
\beq
\sqrt{2/\pi}\int_{\gamma}^\infty t\exp(-\frac{t^2}{2})dt=\sqrt{2/\pi}\exp(-\gamma^2/2).\label{q part 2}
\eeq 
We need to lower bound the number $\gamma>0$. Choose $\delta$ sufficiently small to ensure $\gamma>1$. Using standard lower and upper bounds on the $Q$ function,
\beq
 \frac{1}{\sqrt{2\pi}}\frac{1}{\gamma}\exp(-\frac{{\gamma}^2}{2})\geq Q(\gamma)=\delta\geq \frac{1}{2\sqrt{2\pi}}\frac{1}{\gamma}\exp(-\frac{{\gamma}^2}{2}).\nn
\eeq
The left hand side implies that $\gamma\leq \sqrt{2\log 1/\delta}$ as otherwise the left-hand side would be strictly less than $\delta$. Now, using the right hand side
\beq
\delta \sqrt{2\log 1/\delta}\geq \delta\gamma \geq \frac{1}{2\sqrt{2\pi}}\exp(-\frac{\gamma^2}{2}).\nn
\eeq
This provides the desired upper bound $\exp(-\frac{\gamma^2}{2})\leq C\delta\sqrt{\log\frac{1}{\delta}}$ for $C=4\sqrt{\pi}$. Combining this with \eqref{q part 2}, we can conclude.
\end{proof}
The following lemma is related to the order statistics of a standard Gaussian vector.
\begin{lemma}\label{append k+} Consider the setup of Lemma \ref{gauss k+}. There exists constants $c_1,c_2,c_3$ such that for any $\delta>c_1$, we have that
\beq
\E[\sum_{i=1}^{\delta n} \tilde{\g}_i]\leq c_2\delta n\sqrt{\log \delta^{-1}}\nn
\eeq
whenever $n>c_3\delta^{-1}$.
\end{lemma}
\begin{proof} Let $t$ be the number of entries of $\g$ obeying $|\g_i|\geq \gamma_{2\delta}=Q^{-1}(2\delta)$. First we show that $t$ is around $2\delta n$ with high probability.
\begin{lemma} We have that $\Pro(|t- 2\delta n|\geq \delta n)\leq 2\exp(-\delta n/8)$.
\end{lemma}
\begin{proof} This again follows from a Chernoff bound. In particular $t=\sum_{i=1}^na_i$ where $\Pro(a_i=1)=2\delta$ and $\{a_i\}_{i=1}^n$ are i.i.d. Consequently $\Pro(|\sum_{i=1}^na_i-2\delta n|>t)\leq 2\exp(-\delta n/8)$.
\end{proof}
 Conditioned on $t$, the largest $t$ entries are i.i.d.~and distributed as a standard normal $g\sim \Nn(0,1)$ conditioned on $g\geq Q^{-1}(2\delta)$. Applying Lemma \ref{q tail}, this implies that
 \beq
 \E[\sum_{i=1}^{\delta n}\tilde{g}_i]\leq  \E[\sum_{i=1}^{t}\tilde{g}_i]\leq tc\sqrt{\delta^{-1}}\leq 3cn\delta \sqrt{\log\delta^{-1}}.\label{est1}.
 \eeq
 The remaining event occurs with probability $2\exp(-\delta n/8)$. On this event, for any $t\geq 0$, we have that
 \beq
 \Pro(\tilde{\g}_1\geq \sqrt{2\log n}+t)\leq \exp(\delta n/8)\exp(-t^2/2)\label{est2}
 \eeq
which implies $\E[\tilde{\g}_1]\leq c'(\sqrt{\log n}+\sqrt{\delta n})$ and $ \E[\sum_{i=1}^{\delta n}\tilde{g}_i]\leq c'\delta n(\sqrt{\log n}+\sqrt{\delta n})$. Combining the two estimates \eqref{est1} and \eqref{est2} yields
\begin{align}
\E[\sum_{i=1}^{\delta n}\tilde{\g}_i]&\leq 3cn\delta \sqrt{\log\delta^{-1}}+c'\delta n(\sqrt{\log n}+\sqrt{\delta n}) \exp(-\delta n/8).\nn
\end{align}
Setting $n=\beta \delta^{-1}$, we obtain 
\beq
\E[\frac{1}{\delta n}\sum_{i=1}^{\delta n}\tilde{g}_i]\leq 3c \sqrt{\log\delta^{-1}}+c'(\sqrt{\log \delta^{-1}+\log \beta}+\sqrt{\beta}) \exp(-\beta/8).\nn
\eeq
We can ensure that the second term is $\order{\sqrt{\log\delta^{-1}}}$ by picking $\beta$ to be a large constant to conclude with the result.
%
\end{proof}
\begin{lemma} [Slepian variation]\label{slep var} Let $\A\in\R^{m\times n}$ be a standard Gaussian matrix. Let $\g\in\R^n,\h\in\R^m,g$ be standard Gaussian vectors. Given compact sets $\Cc_1\subset\R^n,\Cc_2\subset\R^m$ we have that
\beq
\E[\sup_{\vb\in\Cc_1,\ub\in\Cc_2} \ub^*\A\vb+\tn{\ub}\tn{\vb} g]\leq \E[\sup_{\vb\in\Cc_1,\ub\in\Cc_2} \vb^*\g\tn{\ub}+\ub^*\h\tn{\vb}].\nn
\eeq
\end{lemma}
\begin{proof} Consider the Gaussian processes $f(\vb,\ub)=\ub^*\A\vb+\tn{\ub}\tn{\vb} g$ and $g(\vb,\ub)=\vb^*\g\tn{\ub}+\ub^*\h\tn{\vb}$. We have that
\begin{align}
&\E[f(\vb,\ub)^2]=\E[g(\vb,\ub)^2]=2\tn{\ub}^2\tn{\vb}^2,\nn\\
&\E[f(\vb,\ub)f(\vb',\ub')]-\E[g(\vb,\ub)g(\vb',\ub')]\geq (\tn{\ub}\tn{\ub'}-\li\ub,\ub'\ri)(\tn{\vb}\tn{\vb'}-\li\vb,\vb'\ri)\geq 0.\nn
\end{align}
Consequently Slepian's Lemma yield $\E[\sup_{\vb\in\Cc_1,\ub\in\Cc_2} f]\leq \E[\sup_{\vb\in\Cc_1,\ub\in\Cc_2} g]$ for finite sets $\Cc_1,\Cc_2$. A standard covering argument finishes the proof.
\end{proof}
\end{document}